\documentclass[journal]{IEEEtran}

\usepackage{amsmath,amssymb,amsthm,mathtools}
\usepackage{graphicx}
\usepackage{booktabs}
\usepackage{multirow}
\usepackage{algorithm}
\usepackage{algpseudocode}
\usepackage{cite}
\usepackage{url}
\usepackage{array}
\usepackage{enumitem}
\usepackage{microtype}

\newtheorem{theorem}{Theorem}
\newtheorem{proposition}{Proposition}
\newtheorem{lemma}{Lemma}
\newtheorem{corollary}{Corollary}
\theoremstyle{remark}
\newtheorem{remark}{Remark}

\newcommand{\R}{\mathbb{R}}
\newcommand{\ones}{\mathbf{1}}
\newcommand{\diag}{\operatorname{Diag}}
\newcommand{\KL}{D_{\rm KL}}
\newcommand{\PO}{\mathcal P_{\Omega}}

\title{Exact Rank-Space KL Projection for Shared-Marginal Low-Rank Factors: Application to Doubly Stochastic Clustering}

\author{Enliang Hu \\ {ynel.hu@gmail.com}}

\begin{document}
\maketitle

\begin{abstract}
We study exact Kullback--Leibler (KL) projection for low-rank factorizations whose two nonnegative factors have prescribed row marginals and a shared, learned column marginal. For arbitrary positive row marginals of equal total mass, the joint KL projection reduces exactly to a strictly convex gauge-fixed dual with only $r-1$ effective variables; its Hessian is a sum of categorical covariance terms and admits $O((n+m)r)$ matrix-free Hessian--vector products. The projection theorem is objective-independent. We then specialize this geometry to doubly stochastic (DS) graph learning through $W=U\operatorname{Diag}(g)^{-1}V^\top$, where row-simplex factors with a common column mass induce an exactly DS graph without materializing an $n\times n$ optimization variable. Combined with observed-edge sparse fitting, a stochastic anchor-reduced manifold regularizer, and Bregman backtracking, the resulting mirror-descent method preserves exact feasibility at every accepted step. Under a nonvanishing latent-mass condition, it satisfies sufficient decrease and an $O(1/N)$ mirror-stationarity bound, while strictly positive accumulation points are KKT stationary. Matched clustering experiments show competitive accuracy, feasibility residuals near numerical precision, and favorable anytime behavior without a dense learned graph.
\end{abstract}


\section{Introduction}
\label{sec:intro}
Graph-based clustering converts pairwise relations into a similarity graph and seeks a partition consistent with its geometry. Spectral methods established the importance of graph normalization and Laplacian structure~\cite{ShiMalik2000,NgJordanWeiss2001,Luxburg2007}. A persistent difficulty is degree imbalance: a few high-degree samples can dominate the graph while weakly connected samples contribute little. Doubly stochastic (DS) similarities alleviate this imbalance by requiring a nonnegative graph to have unit row and column sums. Matrix scaling provides the classical computational basis~\cite{Sinkhorn1964,SinkhornKnopp1967,Knight2008}, and DS normalization has long been useful for spectral clustering~\cite{Zass2006}.

The main obstacle is computational. Direct DS graph-learning models optimize an explicit $n\times n$ matrix, which gives a transparent representation but incurs quadratic state and repeated dense updates~\cite{Wang2025SDSGC}. Low-rank formulations avoid this bottleneck. DCD, for example, represents the graph through a single nonnegative factor and achieves linear-in-$n$ factor storage for fixed rank~\cite{Yang2016DCD}. The price is structural: the same factor appears on both sides of the induced similarity, so representation and optimization are self-coupled. General two-factor factorizations remove that coupling, but $UV^\top$ is not doubly stochastic merely because $U,V\ge0$. Penalty enforcement or post-hoc scaling can restore stochasticity approximately or after the update, but then feasibility is no longer built into the factor geometry.

This leaves a specific design gap. We seek a low-rank clustering model that simultaneously provides four properties: exact DS feasibility throughout optimization, freedom for two latent factors to differ, fitting restricted to observed sparse affinities, and scalable manifold regularization. No single property is difficult in isolation; the challenge is to obtain them together without reintroducing an $n\times n$ state or a large constrained subproblem.

We address this gap with the shared-marginal factorization
\begin{equation}
\label{eq:intro-factor}
W=U\diag(g)^{-1}V^\top,
\end{equation}
subject to
\begin{equation}
\label{eq:intro-constraints}
U\ones=V\ones=\ones,
\qquad U^\top\ones=V^\top\ones=g,
\qquad U,V\ge0.
\end{equation}
The constraints in \eqref{eq:intro-constraints} are only linear equalities and nonnegativity, hence the factor feasible set is convex. For every feasible triplet with $g>0$,
\[
W\ones=U\diag(g)^{-1}g=\ones,
\qquad
W^\top\ones=V\diag(g)^{-1}g=\ones,
\]
so exact DS feasibility follows directly from factor feasibility. The shared mass $g$ is learned rather than prescribed, while $U$ and $V$ remain distinct and can encode complementary latent roles.

The factorization itself has important antecedents in low-rank optimal transport. Entropic OT made Sinkhorn scaling a central primitive~\cite{Cuturi2013}; iterative KL/Bregman projections provide a general constrained-transport viewpoint~\cite{Benamou2015,PeyreCuturi2019}; and factored couplings introduced nonnegative factors linked through a common marginal~\cite{Forrow2019,Scetbon2021,Halmos2024FRLC}. Most recently, Jawanpuria and Mishra study the same balanced factor constraints with prescribed row marginals as a Riemannian manifold~\cite{Jawanpuria2026RLOT}. Their balanced nonlinear KL retraction is computed by cyclic Bregman projections, whereas an $(r-1)$-dimensional Schur complement appears in a tangent-space Fisher--Rao projector. We therefore do not claim the shared-marginal factorization, mirror descent, or the mere existence of rank-sized linear systems as new. Our main geometric result concerns the nonlinear KL projection itself: after eliminating row multipliers and the shared column mass, the complete projection is one strictly convex $(r-1)$-dimensional dual. We then specialize this objective-independent operator to sparse DS clustering. LoRD/B-LoRD form a close clustering line, using a class-probability parameter to obtain a convex linear one-factor DS constraint and projected optimization~\cite{Lyu2026LoRD}. Our graph model instead retains two factors, learns their common latent mass, and preserves exact feasibility inside every accepted mirror step.

Two additional choices make the model suitable for sparse graph data. First, absent edges in a $k$-NN graph should not automatically be treated as equally strong negative observations. We therefore fit only the observed support, following the general principle of observed-entry objectives in incomplete-matrix estimation~\cite{Mazumder2010}. Second, local affinities alone may not capture broader geometry. A sparse sample-to-anchor matrix $Z$ defines
\[
L_a=I-ZD_a^{-1}Z^\top,
\]
which is applied without materializing either the reduced similarity or the Laplacian. The row-stochastic anchor construction gives a nonnegative stochastic PSD similarity and the bound $0\preceq L_a\preceq I$, while each action costs $O(nqr)$ when every sample uses at most $q$ anchors~\cite{LiuHeChang2010,ChenCai2011}. Nystr\"om approximation is retained as an empirical alternative~\cite{WilliamsSeeger2001,Fowlkes2004}.

The remaining difficulty is optimization. Although the feasible factor set is convex, the graph objective is nonconvex because $U$, $V$, and $g$ interact multiplicatively. We first step away from clustering and consider arbitrary positive prescribed row marginals $a$ and $b$ of equal total mass. The exact joint KL projection onto the corresponding shared-marginal factor polytope has only $r-1$ effective dual variables after gauge fixing; its Hessian is a sum of weighted covariance terms and is positive definite on the reduced space. The same dual admits matrix-free Hessian--vector products in $O((n+m)r)$ time. This theorem depends only on the factor constraints and KL geometry, not on the clustering loss or anchor regularizer. The DS clustering feasible set is the special case $a=b=\mathbf1$, so the same operator preserves exact feasibility inside every accepted mirror step rather than repairing the graph afterward.

The contributions are as follows.
\begin{itemize}[leftmargin=1.25em]
\item \textbf{Exact DS structure in a convex two-factor geometry.} Proposition~\ref{prop:shared-marginal-geometry} formalizes the model-level property: linear row-simplex/shared-marginal constraints define a convex factor polytope, and every feasible triplet with positive latent mass induces an exactly doubly stochastic graph without imposing $U=V$. The same representation admits an optional symmetric DS output of rank at most $2r$.
\item \textbf{An objective-independent rank-space KL projection.} Theorem~\ref{thm:uv-low-dimensional-kl} treats arbitrary positive row marginals of equal total mass and shows that the complete joint KL projection reduces to an $(r-1)$-dimensional strictly convex dual. The reduced Hessian is SPD after gauge fixing; explicit Newton costs $O((n+m)r^2+r^3)$ per evaluation/solve, while a Hessian--vector product costs only $O((n+m)r)$.
\item \textbf{Feasible descent with a quantitative stationarity guarantee.} Theorem~\ref{thm:uvg-stationarity} proves finite backtracking, exact feasibility, Bregman sufficient decrease, asymptotic regularity, and an $O(1/N)$ mirror-stationarity bound. Corollary~\ref{cor:uvg-kkt} then identifies every strictly positive accumulation point as KKT stationary under the same nonvanishing-mass condition.
\item \textbf{Sparse manifold modeling and controlled empirical comparison.} Observed-edge fitting and a stochastic anchor-reduced Laplacian avoid dense learned graph states, while the experiments compare representative full-graph, one-factor low-rank, and shared-marginal alternatives under a unified clean/matched protocol. Ablations isolate the reduced geometry, factor-side choice, exact projection, and stationarity behavior.
\end{itemize}

The main text gives complete theorem statements together with concise proof sketches. Appendix~\ref{app:kl-proof} contains the full proof of the exact rank-space KL projection theorem; longer proofs of the auxiliary structural and convergence statements are omitted for brevity.

\section{Related Work}
\label{sec:related}
\subsection{Spectral and doubly stochastic graph clustering}
Spectral clustering links partitioning to the eigenspaces of graph Laplacians. Normalized cuts~\cite{ShiMalik2000}, the normalized spectral algorithm of Ng, Jordan, and Weiss~\cite{NgJordanWeiss2001}, and later tutorials~\cite{Luxburg2007} established the main normalization principles used in graph clustering. DS normalization goes further by balancing both row and column masses. Its computational basis is matrix scaling~\cite{Sinkhorn1964,SinkhornKnopp1967}, for which convergence and numerical properties have been studied extensively~\cite{Knight2008}. Zass and Shashua used DS approximation directly as a spectral-clustering normalization~\cite{Zass2006}. Learned-graph formulations subsequently imposed DS and clustering structure jointly; a KDD formulation learned a structured DS matrix with a low-rank Laplacian constraint~\cite{WangNieHuang2016}, while the more recent SDSGC combines a full DS graph with spectral structure~\cite{Wang2025SDSGC}. These direct models retain an interpretable graph variable but generally require $O(n^2)$ state storage.

\subsection{Low-rank graph factorization for clustering}
Low-rank nonnegative factorizations are closely related to clustering and spectral objectives~\cite{DingHeSimon2005,KuangDingPark2012}. DCD specializes this idea to DS clustering through a Data--Cluster--Data random walk and a one-factor nonnegative low-rank decomposition~\cite{Yang2016DCD}. Its storage is linear in $n$ for fixed rank, but the repeated factor in the induced similarity creates self-coupling and leads naturally to multiplicative MM updates. More recently, LoRD and B-LoRD revisit low-rank DS clustering from a kernel $k$-means perspective and introduce a class-probability parameter that turns their one-factor DS constraint into a convex linear factor constraint~\cite{Lyu2026LoRD}. Our setting is different: two factors share a learned latent mass, the factors need not coincide, and feasibility is enforced by a joint KL projection whose effective dimension is $r-1$.

\subsection{Low-rank transport, matrix scaling, and Bregman projections}
Entropic OT connects transport computation to Sinkhorn scaling~\cite{Cuturi2013,PeyreCuturi2019}, while iterative Bregman projections provide a general KL-projection framework for intersections of affine marginal constraints~\cite{Benamou2015}. Factored couplings introduced a common intermediate marginal~\cite{Forrow2019}; low-rank Sinkhorn later optimized such factors with LR-Dykstra-style KL projections and proved non-asymptotic stationarity~\cite{Scetbon2021}. FRLC uses a more flexible latent-coupling factorization and coordinate mirror descent~\cite{Halmos2024FRLC}.

The closest geometric antecedent is the 2026 Riemannian formulation of balanced low-rank OT~\cite{Jawanpuria2026RLOT}. It uses the same constraint pattern $U\mathbf1=a$, $V\mathbf1=b$, and $U^\top\mathbf1=V^\top\mathbf1$. Its balanced nonlinear KL retraction is computed through cyclic Bregman projections over the three affine constraint families; separately, its tangent-space Fisher--Rao projector reduces to an $(r-1)$-dimensional Schur-complement system. Our novelty boundary is narrower and different: Theorem~\ref{thm:uv-low-dimensional-kl} eliminates the row multipliers and shared mass in the \emph{nonlinear KL projection itself}, yielding one strictly convex $(r-1)$-dimensional dual whose minimizer is the exact joint projection. Thus we claim neither shared-marginal factors nor an $(r-1)$ tangent linear system in isolation. The clustering contribution is the use of this direct projection as the finite-iterate feasibility operator for a sparse nonconvex DS graph objective.

\subsection{Mirror descent and non-Euclidean first-order methods}
Mirror descent replaces Euclidean projection by a geometry induced by a distance-generating function; entropy is especially natural on simplices~\cite{BeckTeboulle2003}. Bregman proximal-gradient analyses extend descent arguments beyond ordinary Euclidean Lipschitz-gradient settings~\cite{BauschkeBolteTeboulle2017}, while relative smoothness provides a general first-order framework for non-Euclidean descent~\cite{LuFreundNesterov2018}. Our convergence proof is self-contained and uses a problem-specific Bregman sufficient-decrease line search, but these works provide the optimization context for the KL mirror map and the use of non-Euclidean descent certificates.

\begin{table*}[t]
\centering
\caption{Closest shared-marginal projection mechanisms. The comparison concerns how marginal feasibility is restored, not the surrounding learning objective.}
\label{tab:projection-related}
\small
\setlength{\tabcolsep}{5pt}
\begin{tabular}{p{3.2cm}p{4.0cm}p{4.2cm}p{3.0cm}}
\toprule
Method & factor constraints & nonlinear KL feasibility step & rank-space structure \\
\midrule
Low-rank Sinkhorn~\cite{Scetbon2021} & prescribed transport marginals + shared $g$ & LR-Dykstra / alternating KL projections & $g\in\R_+^r$ kept explicitly \\
Riemannian low-rank OT~\cite{Jawanpuria2026RLOT} & $U\ones=a$, $V\ones=b$, shared column sums & cyclic Bregman projection for balanced retraction & $(r-1)$ Schur system in tangent projector \\
\textbf{This work} & same general balanced factor polytope; DS is $a=b=\ones$ & \textbf{one exact joint KL projection via a strictly convex dual} & \textbf{$r-1$ nonlinear dual variables; HVP $O((n+m)r)$} \\
\bottomrule
\end{tabular}
\end{table*}

\subsection{Sparse fitting and reduced manifold graphs}
Sparse input similarities naturally suggest losses that distinguish observed edges from unobserved pairs. The use of observed-support objectives has a broad precedent in incomplete-matrix estimation, including Soft-Impute-type methods~\cite{Mazumder2010}; we use this principle only as motivation and do not rely on nuclear-norm completion. For scalable manifold modeling, anchor graphs replace full pairwise relations by sparse sample-to-anchor weights and can yield nonnegative PSD graph operators with linear-scale storage~\cite{LiuHeChang2010}. Landmark-based spectral clustering similarly uses a small representative set to obtain linear scaling in the number of samples~\cite{ChenCai2011}. Nystr\"om approximation is another classical route to low-rank kernel and spectral computations~\cite{WilliamsSeeger2001,Fowlkes2004}. Our main method uses the stochastic anchor construction because its reduced similarity is nonnegative, PSD, and exactly stochastic by construction; Nystr\"om is retained as an accuracy-oriented alternative in the ablation study.

The closest methods obtain stochastic structure in different ways and therefore pay different computational or modeling costs. Table~\ref{tab:position} separates the main structural axes. The comparison is intentionally restricted to representative DS formulations and to low-rank Sinkhorn, the closest shared-marginal antecedent; generic NMF and spectral baselines are retained for the empirical study rather than used to define the methodological gap.

\begin{table*}[t]
\centering
\caption{Structural mechanisms in representative doubly stochastic and shared-marginal formulations. ``Exact'' refers to the represented marginal constraints and does not imply global optimality. Low-rank Sinkhorn is an optimal-transport antecedent rather than a clustering baseline.}
\label{tab:position}
\scriptsize
\setlength{\tabcolsep}{3.0pt}
\renewcommand{\arraystretch}{1.15}
\begin{tabular}{p{1.75cm}p{2.05cm}p{2.55cm}p{1.25cm}p{1.25cm}p{1.95cm}p{1.45cm}p{3.05cm}}
\toprule
Method & Learned state & DS / marginal mechanism & Symmetry by model & Two-factor freedom & Latent mass & Sparse-edge objective & Structure beyond stochasticity \\
\midrule
SDSGC~\cite{Wang2025SDSGC}
& full $W$ and spectral factor
& explicit DS constraints optimized jointly by ALM
& Yes
& No
& --
& No
& structured Laplacian / graph connectivity \\
DCD~\cite{Yang2016DCD}
& one factor $P\in\R_+^{n\times r}$
& DS by normalized DCD parameterization
& Yes
& No
& implicit cluster mass
& Yes
& sparse-similarity KL fitting \\
LoRD~\cite{Lyu2026LoRD}
& one factor $V\in\R_+^{n\times r}$
& linear convex factor constraints + projection
& Yes
& No
& prescribed class prior $\mu$
& No$^{\dagger}$
& probabilistic clustering \\
B-LoRD~\cite{Lyu2026LoRD}
& one factor $V\in\R_+^{n\times r}$
& linear convex factor constraints + projection
& Yes
& No
& prescribed class prior $\mu$
& No$^{\dagger}$
& adjustable block-diagonality regularization \\
Low-rank Sinkhorn~\cite{Scetbon2021}
& $Q,R,g$
& exact transport marginals via KL/Dykstra projections
& No, in general
& Yes
& learned shared $g$
& No
& transport-cost geometry \\
\textbf{SMDSC}
& $U,V,g$ plus sparse anchors
& exact DS at every accepted iterate via joint KL projection
& Not forced; optional symmetrization
& Yes
& learned shared $g$
& Yes
& stochastic anchor-reduced manifold \\
\bottomrule
\end{tabular}
\vspace{1mm}
\parbox{0.97\textwidth}{\footnotesize $^{\dagger}$LoRD/B-LoRD can exploit a sparse input $S$ computationally, but their Frobenius objective is not an observed-edge-only loss. For SMDSC, $W$ need not be symmetric because the two factors are deliberately allowed to differ. If a symmetric graph is required, $W_s=(W+W^\top)/2$ remains nonnegative and doubly stochastic, with $\operatorname{rank}(W_s)\le 2r$.}
\end{table*}

Table~\ref{tab:position} highlights the intended niche of SMDSC. SDSGC keeps rich structure in an explicit graph; DCD and LoRD/B-LoRD obtain a compact symmetric graph through a single factor; low-rank Sinkhorn has two factors and a learned shared marginal but optimizes an OT cost. SMDSC combines the two-factor shared-marginal geometry with an observed-edge clustering loss and an anchor-reduced manifold term. The distinction is not that the other methods lack stochastic structure, but that they obtain it through a different representation or pay for it in a different part of the algorithm.

\section{Proposed Method}
\label{sec:method}

\subsection{Shared-marginal low-rank doubly stochastic factorization}
\label{sec:shared-marginal-model}
Shared-marginal factorizations are also used in low-rank optimal transport~\cite{Forrow2019,Scetbon2021}; here we use the geometry for sparse graph clustering rather than transport-cost minimization. Let $U,V\in\mathbb R_+^{n\times r}$ and $g\in\mathbb R_+^r$. We parameterize the learned graph as
\begin{equation}
\label{eq:uvg-graph}
W(U,V,g):=U\operatorname{Diag}(g)^{-1}V^\top,
\end{equation}
subject to
\begin{equation}
\label{eq:uvg-marginals}
U\mathbf 1_r=V\mathbf 1_r=\mathbf 1_n,
\qquad
U^\top\mathbf 1_n=V^\top\mathbf 1_n=g.
\end{equation}
\begin{proposition}[Convex shared-marginal geometry and exact DS feasibility]
\label{prop:shared-marginal-geometry}
Let
\begin{equation}
\label{eq:triplet-feasible-set}
\begin{aligned}
\mathcal F:=\{(U,V,g)\ge0:\;&
U\mathbf1_r=V\mathbf1_r=\mathbf1_n,\\
&U^\top\mathbf1_n=V^\top\mathbf1_n=g\}.
\end{aligned}
\end{equation}
Then $\mathcal F$ is a convex polytope. For every $(U,V,g)\in\mathcal F$ with $g>0$, the matrix $W=U\operatorname{Diag}(g)^{-1}V^\top$ is nonnegative and doubly stochastic, with $\operatorname{rank}(W)\le r$. Moreover,
\begin{equation}
\label{eq:sym-output}
W_s:=\frac12(W+W^\top)
\end{equation}
is nonnegative, symmetric, and doubly stochastic, with $\operatorname{rank}(W_s)\le2r$.
\end{proposition}

\begin{proof}
The defining constraints are affine equalities plus nonnegativity, so $\mathcal F$ is convex and polyhedral. The row-simplex constraints give $0\le U_{ij},V_{ij}\le1$, and hence $0\le g_j\le n$; therefore $\mathcal F$ is bounded and is a convex polytope. For $g>0$, nonnegativity of $W$ is immediate and
\[
\begin{aligned}
W\mathbf1_n&=U\operatorname{Diag}(g)^{-1}g=\mathbf1_n,\\
W^\top\mathbf1_n&=V\operatorname{Diag}(g)^{-1}g=\mathbf1_n.
\end{aligned}
\]
The rank bound $\operatorname{rank}(W)\le r$ follows from the factorization. Averaging $W$ with $W^\top$ preserves nonnegativity and both stochastic marginals, and subadditivity of rank gives $\operatorname{rank}(W_s)\le2r$.
\end{proof}

For optimization it is convenient to eliminate $g$. Define
\begin{equation}
\label{eq:uv-feasible-set}
\begin{aligned}
\mathcal C:=\{(U,V):\;&U,V\in\mathbb R_+^{n\times r},\;
U\mathbf1_r=V\mathbf1_r=\mathbf1_n,\\
&U^\top\mathbf1_n=V^\top\mathbf1_n\}.
\end{aligned}
\end{equation}
and extend the shared latent mass smoothly to a neighborhood of $\mathcal C$ by
\begin{equation}
\label{eq:gbar}
\bar g(U,V):=\frac12\left(U^\top\mathbf1_n+V^\top\mathbf1_n\right),
\qquad
D_{\bar g}:=\operatorname{Diag}(\bar g).
\end{equation}
On $\mathcal C$, $\bar g=U^\top\mathbf1_n=V^\top\mathbf1_n$.

\subsection{Anchor-reduced manifold regularization}
\label{sec:anchor-reduced-laplacian}
Following scalable anchor/landmark graph constructions~\cite{LiuHeChang2010,ChenCai2011}, let $Z\in\mathbb R_+^{n\times m}$ be a sparse sample-to-anchor matrix with normalized rows,
\begin{equation}
Z\mathbf1_m=\mathbf1_n.
\end{equation}
Define
\begin{equation}
\label{eq:anchor-operator}
\begin{aligned}
d_a&:=Z^\top\mathbf1_n, & D_a&:=\operatorname{Diag}(d_a),\\
S_a&:=ZD_a^{-1}Z^\top, & L_a&:=I_n-S_a.
\end{aligned}
\end{equation}
where zero-mass anchors are removed so that $d_a>0$. The matrix $S_a$ is never formed explicitly.

\begin{proposition}[Structure of the anchor-reduced Laplacian]
\label{prop:anchor-laplacian}
Suppose $Z\ge0$, $Z\mathbf1_m=\mathbf1_n$, and $d_a=Z^\top\mathbf1_n>0$. Then
\begin{equation}
S_a=S_a^\top,\qquad S_a\ge0,\qquad S_a\succeq0,
\qquad S_a\mathbf1_n=\mathbf1_n.
\end{equation}
Moreover,
\begin{equation}
0\preceq S_a\preceq I_n,
\qquad
0\preceq L_a\preceq I_n,
\qquad
L_a\mathbf1_n=0.
\end{equation}
Consequently, for
\begin{equation}
\phi(V):=\frac{\mu}{2}\operatorname{tr}(V^\top L_aV),
\end{equation}
we have
\begin{equation}
\nabla\phi(V)=\mu L_aV,
\qquad
\|\nabla\phi(V_1)-\nabla\phi(V_2)\|_F
\le \mu\|V_1-V_2\|_F.
\end{equation}
\end{proposition}

The anchor action is evaluated without forming an $n\times n$ matrix:
\begin{equation}
\label{eq:anchor-action}
L_aV=V-Z\left[D_a^{-1}(Z^\top V)\right].
\end{equation}
If each sample connects to at most $q$ anchors, then $\operatorname{nnz}(Z)\le nq$, so one anchor action costs $\mathcal O(nqr)$ operations and $\mathcal O(nq)$ additional storage.

\subsection{Reduced sparse-fitting objective and gradients}
\label{sec:uvg-objective}
Let $A=A^\top\ge0$ be a sparse affinity matrix with symmetric support $\Omega$. Our main reduced objective is
\begin{equation}
\label{eq:uvg-reduced-objective}
\begin{aligned}
\Phi(U,V):={}&\frac12\left\|\mathcal P_\Omega\!\left(UD_{\bar g}^{-1}V^\top-A\right)\right\|_F^2\\
&+\frac\mu2\operatorname{tr}(V^\top L_aV),
\qquad (U,V)\in\mathcal C.
\end{aligned}
\end{equation}
The factor $V$ is designated as the manifold-aware factor. This does not privilege a particular mathematical side: Proposition~\ref{prop:factor-swap} shows that regularizing $U$ instead gives the factor-swapped problem.

Let
\begin{equation}
E:=\mathcal P_\Omega(W-A),
\qquad
g:=\bar g(U,V),
\qquad
D:=\operatorname{Diag}(g),
\end{equation}
and define $h\in\mathbb R^r$ componentwise by
\begin{equation}
\label{eq:g-derivative}
h_j:=-\frac{[U^\top E V]_{jj}}{g_j^2}.
\end{equation}
A direct chain-rule calculation gives
\begin{equation}
\label{eq:uv-reduced-gradients}
\begin{aligned}
G_U(U,V):=\nabla_U\Phi(U,V)
&=EVD^{-1}+\frac12\mathbf1_n h^\top,\\
G_V(U,V):=\nabla_V\Phi(U,V)
&=E^\top UD^{-1}+\mu L_aV+\frac12\mathbf1_n h^\top.
\end{aligned}
\end{equation}
Thus the equal split of the $g$-derivative is exactly the chain rule induced by \eqref{eq:gbar}; it is not an auxiliary heuristic.

\begin{proposition}[Factor-swap equivalence]
\label{prop:factor-swap}
Assume $A=A^\top$ and $\Omega=\Omega^\top$. Define
\begin{align}
\Phi_V(U,V)
&:=f_\Omega\!\left(UD_g^{-1}V^\top;A\right)
+\frac\mu2\operatorname{tr}(V^\top L_aV),\\
\Phi_U(U,V)
&:=f_\Omega\!\left(UD_g^{-1}V^\top;A\right)
+\frac\mu2\operatorname{tr}(U^\top L_aU),
\end{align}
where $f_\Omega(W;A)=\frac12\|\mathcal P_\Omega(W-A)\|_F^2$ and $(U,V,g)$ satisfies \eqref{eq:uvg-marginals}. Then
\begin{equation}
\Phi_V(U,V)=\Phi_U(V,U).
\end{equation}
Consequently, the $V$-regularized and $U$-regularized problems have the same optimal value, and their feasible stationary points correspond under $(U,V)\mapsto(V,U)$.
\end{proposition}

\begin{proposition}[Symmetric manifold regularization preserves the one-factor submanifold]
\label{prop:symmetric-one-factor}
Assume $A=A^\top$, $\Omega=\Omega^\top$, and $L_a=L_a^\top$. Consider
\begin{equation}
\begin{aligned}
\Phi_{\rm sym}(U,V)={}&f_\Omega(W(U,V);A)\\
&+\frac\mu4\operatorname{tr}(U^\top L_aU)
+\frac\mu4\operatorname{tr}(V^\top L_aV).
\end{aligned}
\end{equation}
If an exact mirror-descent iteration is initialized with $U_0=V_0$, then
\begin{equation}
U_k=V_k\qquad\text{for all }k.
\end{equation}
Hence the iteration remains on the one-factor submanifold
\begin{equation}
W_k=U_k\operatorname{Diag}(U_k^\top\mathbf1_n)^{-1}U_k^\top.
\end{equation}
\end{proposition}

\subsection{Exact rank-space KL projection}
\label{sec:exact-kl-projection}
Before specializing to clustering, consider two factors with general prescribed row marginals. Let $a\in\R_{++}^{n}$ and $b\in\R_{++}^{m}$ satisfy
\[
\ones_n^\top a=\ones_m^\top b=:M,
\]
and define
\begin{equation}
\label{eq:general-shared-polytope}
\begin{aligned}
\mathcal C(a,b)=\{(U,V)\ge0:
&\ U\ones_r=a,\ V\ones_r=b,\\
&\ U^\top\ones_n=V^\top\ones_m\}.
\end{aligned}
\end{equation}
The clustering feasible set in \eqref{eq:uv-feasible-set} is the special case $m=n$ and $a=b=\ones_n$.

Given positive references $\bar U\in\R_{++}^{n\times r}$ and $\bar V\in\R_{++}^{m\times r}$, consider
\begin{equation}
\label{eq:uv-kl-projection}
\min_{(U,V)\in\mathcal C(a,b)}\KL(U\|\bar U)+\KL(V\|\bar V).
\end{equation}
For $z\in\R^r$, define rowwise
\begin{equation}
\label{eq:uv-softmax-general}
u_i(z)=a_i\frac{\bar u_i\odot e^z}{\bar u_i^\top e^z},\qquad
v_\ell(z)=b_\ell\frac{\bar v_\ell\odot e^{-z}}{\bar v_\ell^\top e^{-z}},
\end{equation}
and let $U(z),V(z)$ stack these rows. Set
\begin{equation}
\label{eq:uv-dual-potential}
\psi_{a,b}(z)=\sum_{i=1}^{n}a_i\log(\bar u_i^\top e^z)
+\sum_{\ell=1}^{m}b_\ell\log(\bar v_\ell^\top e^{-z}).
\end{equation}

\begin{theorem}[Exact shared-marginal KL projection]
\label{thm:uv-low-dimensional-kl}
Let $A_a=\operatorname{Diag}(a)$, $B_b=\operatorname{Diag}(b)$,
$c_U(z)=U(z)^\top\ones_n$, and $c_V(z)=V(z)^\top\ones_m$. Then
\begin{align}
\nabla\psi_{a,b}(z)&=c_U(z)-c_V(z),\label{eq:uv-dual-gradient}\\
\nabla^2\psi_{a,b}(z)
&=\operatorname{Diag}(c_U)-U^\top A_a^{-1}U \nonumber\\
&\quad+\operatorname{Diag}(c_V)-V^\top B_b^{-1}V\succeq0.\label{eq:uv-dual-hessian}
\end{align}
The nullspace of \eqref{eq:uv-dual-hessian} is exactly $\operatorname{span}\{\ones_r\}$ and
$\psi_{a,b}(z+c\ones_r)=\psi_{a,b}(z)$. After fixing the gauge $z_r=0$, $\psi_{a,b}$ is strictly convex and coercive, hence has a unique minimizer $z^\star$. The unique solution of \eqref{eq:uv-kl-projection} is $(U(z^\star),V(z^\star))$. Thus the exact projection has only $r-1$ effective unknowns; one Hessian evaluation costs $O((n+m)r^2)$ and the reduced Newton solve costs $O(r^3)$. For the clustering case $a=b=\ones_n$, \eqref{eq:uv-dual-hessian} reduces to
\[
\operatorname{Diag}(c_U)-U^\top U+\operatorname{Diag}(c_V)-V^\top V.
\]
\end{theorem}

The dual also admits matrix-free curvature products. For any $x\in\R^r$,
\begin{align}
\label{eq:uv-hvp}
\nabla^2\psi_{a,b}(z)x
&=\operatorname{Diag}(c_U)x-U^\top A_a^{-1}(Ux)\nonumber\\
&\quad+\operatorname{Diag}(c_V)x-V^\top B_b^{-1}(Vx),
\end{align}
which costs $O((n+m)r)$ without forming the Hessian. Thus large-rank variants can use Newton--CG on the gauge-fixed system, whereas the direct Cholesky route in Theorem~\ref{thm:uv-low-dimensional-kl} is preferable at the small ranks typical of clustering.

\emph{Proof sketch.}
The scaled row normalizations in \eqref{eq:uv-softmax-general} enforce the prescribed row marginals. Differentiating \eqref{eq:uv-dual-potential} gives \eqref{eq:uv-dual-gradient}; the Hessian is a positive weighted sum of row-wise categorical covariance matrices. Strict positivity makes their common nullspace exactly the all-ones direction. Equality of the total row masses gives shift invariance, and gauge fixing removes the only null direction. The remaining KKT equation is $c_U=c_V$, precisely the shared-column constraint. Appendix~\ref{app:kl-proof} gives the complete proof.

This result is objective-independent: any KL-mirror or Bregman method whose feasible factors have prescribed row marginals and a shared column marginal can use the same rank-space projection. It concerns the complete nonlinear KL projection itself, in contrast to cyclic Bregman realization of the balanced retraction and the separate rank-sized Schur system in the tangent-space projector of~\cite{Jawanpuria2026RLOT}. In the remainder we specialize to $a=b=\ones_n$.

\subsection{Feasible mirror-descent algorithm}
\label{sec:uvg-mirror-algorithm}
Given a feasible positive pair $(U_k,V_k)$ and step size $\eta>0$, form the exponentiated references
\begin{equation}
\label{eq:uv-exp-trial}
\bar U_k=U_k\odot\exp(-\eta G_U^k),
\qquad
\bar V_k=V_k\odot\exp(-\eta G_V^k),
\end{equation}
where $G_U^k=G_U(U_k,V_k)$ and $G_V^k=G_V(U_k,V_k)$. The next candidate is the exact KL projection of $(\bar U_k,\bar V_k)$ onto $\mathcal C$, computed by Theorem~\ref{thm:uv-low-dimensional-kl}. Equivalently,
\begin{equation}
\label{eq:uv-mirror-map}
\begin{aligned}
\mathcal T_\eta(U,V):=\arg\min_{(\widetilde U,\widetilde V)\in\mathcal C}
\Big\{&\langle G_U,\widetilde U-U\rangle_F\\
&+\langle G_V,\widetilde V-V\rangle_F\\
&+\frac1\eta D_{\rm KL}(\widetilde U\|U)\\
&+\frac1\eta D_{\rm KL}(\widetilde V\|V)\Big\},
\end{aligned}
\end{equation}
where $G_U=G_U(U,V)$ and $G_V=G_V(U,V)$.
Write $(U^+,V^+)=\mathcal T_\eta(U,V)$. A backtracking step is accepted when
\begin{equation}
\label{eq:uv-bregman-sd}
\begin{aligned}
\Phi(U^+,V^+)\le{}&\Phi(U,V)\\
&+\langle G_U,U^+-U\rangle_F\\
&+\langle G_V,V^+-V\rangle_F\\
&+\frac{1-\sigma}{\eta}D_{\rm KL}(U^+\|U)\\
&+\frac{1-\sigma}{\eta}D_{\rm KL}(V^+\|V),
\end{aligned}
\end{equation}
where $\sigma\in(0,1)$.
The line search starts from a fixed $\bar\eta>0$ at each outer iteration and multiplies it by $\rho\in(0,1)$ until \eqref{eq:uv-bregman-sd} holds.

\begin{lemma}[Mirror fixed points characterize first-order stationarity]
\label{lem:uv-mirror-fixed-point}
Let $(U,V)\in\mathcal C\cap\mathbb R_{++}^{2nr}$ and $\eta>0$. Then
\begin{equation}
\begin{aligned}
\mathcal T_\eta(U,V)=(U,V)
\quad\Longleftrightarrow\quad\;&
0\in (G_U(U,V),G_V(U,V))\\
&+N_{\mathcal C}(U,V).
\end{aligned}
\end{equation}
Equivalently, if $(U^+,V^+)=\mathcal T_\eta(U,V)$, then
\begin{equation}
\label{eq:uv-mirror-residual}
\frac{\mathcal D_{\rm sym}(U,V;U^+,V^+)}{\eta^2},
\end{equation}
where
\begin{align*}
\mathcal D_{\rm sym}:={}&D_{\rm KL}(U^+\|U)+D_{\rm KL}(U\|U^+)\\
&+D_{\rm KL}(V^+\|V)+D_{\rm KL}(V\|V^+).
\end{align*}
vanishes if and only if $(U,V)$ is first-order stationary.
\end{lemma}

\begin{theorem}[Feasible descent and mirror-stationarity]
\label{thm:uvg-stationarity}
Assume:
\begin{enumerate}
\item $L_a\succeq0$ and $\|L_a\|_2\le1$;
\item $(U_0,V_0)\in\mathcal C\cap\mathbb R_{++}^{2nr}$;
\item the accepted iterates have nonvanishing latent masses: there exists $\underline g>0$ such that
\begin{equation}
\min_j \bar g_j(U_k,V_k)\ge\underline g\qquad\text{for all }k;
\end{equation}
\item every KL projection is solved exactly.
\end{enumerate}
Then the backtracking mirror iteration defined by \eqref{eq:uv-mirror-map}--\eqref{eq:uv-bregman-sd} has the following properties.

\textbf{(i) Finite backtracking and a uniform positive step-size lower bound.}
There exists $\eta_{\min}>0$ such that every accepted step satisfies
\begin{equation}
0<\eta_{\min}\le\eta_k\le\bar\eta.
\end{equation}

\textbf{(ii) Exact feasibility, positivity, and sufficient decrease.}
Every finite iterate belongs to $\mathcal C\cap\mathbb R_{++}^{2nr}$. Moreover,
\begin{equation}
\label{eq:uv-strong-decrease}
\begin{aligned}
&\Phi(U_k,V_k)-\Phi(U_{k+1},V_{k+1})\\
&\ge \frac{\sigma}{\eta_k}D_{\rm KL}(U_{k+1}\|U_k)
+\frac{\sigma}{\eta_k}D_{\rm KL}(V_{k+1}\|V_k)\\
&\quad+\frac1{\eta_k}D_{\rm KL}(U_k\|U_{k+1})
+\frac1{\eta_k}D_{\rm KL}(V_k\|V_{k+1}).
\end{aligned}
\end{equation}

\textbf{(iii) Asymptotic regularity.}
The objective values converge and
\begin{equation}
\sum_{k=0}^{\infty}
\left(\|U_{k+1}-U_k\|_F^2+\|V_{k+1}-V_k\|_F^2\right)<\infty.
\end{equation}
Consequently,
\begin{equation}
\|U_{k+1}-U_k\|_F+\|V_{k+1}-V_k\|_F\to0.
\end{equation}

\textbf{(iv) Non-asymptotic mirror-stationarity rate.}
Define
\begin{equation}
\label{eq:uv-stationarity-measure}
\begin{aligned}
\Delta_k:=\frac{1}{\eta_k^2}\Big[
&D_{\rm KL}(U_{k+1}\|U_k)+D_{\rm KL}(U_k\|U_{k+1})\\
+&D_{\rm KL}(V_{k+1}\|V_k)+D_{\rm KL}(V_k\|V_{k+1})\Big].
\end{aligned}
\end{equation}
Then, for every $N\ge1$,
\begin{equation}
\label{eq:uv-stationarity-rate}
\min_{0\le k<N}\Delta_k
\le
\frac{\Phi(U_0,V_0)-\Phi_{\inf}}
{\sigma\eta_{\min}N},
\end{equation}
where $\Phi_{\inf}\ge0$ is the infimum of $\Phi$ over its positive-mass feasible domain. In particular, $\Delta_k\to0$.

\end{theorem}

\paragraph*{Proof sketch}
The nonvanishing-mass assumption places every accepted iterate in a compact subset of $\mathcal C$ on which the reduced gradient is uniformly Lipschitz. Comparing the mirror subproblem with the current feasible point and using strong convexity of negative entropy shows that sufficiently small trials remain in a slightly larger positive-mass region; the descent lemma then guarantees acceptance, giving finite backtracking and a uniform positive step-size lower bound. Mirror optimality combined with the acceptance test yields the symmetrized-KL decrease inequality. Summation gives asymptotic regularity and the $O(1/N)$ residual bound. The longer technical derivation is omitted for brevity.

\begin{corollary}[KKT stationarity of positive accumulation points]
\label{cor:uvg-kkt}
Under the assumptions of Theorem~\ref{thm:uvg-stationarity}, every accumulation point $(U^\star,V^\star)$ whose entries are strictly positive satisfies
\begin{equation}
\label{eq:uv-normal-stationarity}
0\in (G_U(U^\star,V^\star),G_V(U^\star,V^\star))
+N_{\mathcal C}(U^\star,V^\star).
\end{equation}
Equivalently, there exist row multipliers $\alpha,\beta\in\mathbb R^n$ and a shared-marginal multiplier $\lambda\in\mathbb R^r$ such that
\begin{equation}
\label{eq:uv-kkt-positive}
\begin{aligned}
G_U(U^\star,V^\star)+\alpha\mathbf1_r^\top+\mathbf1_n\lambda^\top&=0,\\
G_V(U^\star,V^\star)+\beta\mathbf1_r^\top-\mathbf1_n\lambda^\top&=0,
\end{aligned}
\end{equation}
together with the feasibility equations defining $\mathcal C$.
\end{corollary}

\begin{remark}[Vanishing latent components]
\label{rem:vanishing-g}
The nonvanishing-mass assumption is needed because $D_g^{-1}$ becomes singular when $g_j\to0$. The present algorithm does not include a barrier that guarantees a uniform lower bound on every latent mass, so Theorem~\ref{thm:uvg-stationarity} is explicitly conditional on this property. If a latent mass tends to zero, the corresponding rank-one component carries vanishing total mass in the shared-marginal factorization and can be removed, after which the analysis applies to the active reduced-rank representation. In implementation, $\min_j g_j$ should therefore be monitored explicitly.
\end{remark}

\begin{remark}[Scope of novelty]
Shared-marginal low-rank coupling factorizations and KL mirror descent have antecedents in low-rank optimal transport. The present contribution should therefore not claim the parameterization itself as new. The paper-specific theoretical contribution is the specialization to sparse doubly stochastic graph fitting with anchor-reduced manifold regularization, the $(r-1)$-dimensional exact KL projection characterization, and the resulting feasible Bregman descent and stationarity analysis.
\end{remark}

\subsection{Practical anchor construction and parameter scaling}
\label{sec:practical}
The theoretical results treat $L_a$ and $\mu$ as fixed. In the implementation, $m$ anchors are selected from the feature space, each sample is connected to a small number $q$ of nearby anchors, and the resulting nonnegative weights are normalized row-wise to obtain $Z\ones=\ones$. This gives the operator in \eqref{eq:anchor-operator}.

The strength of an anchor manifold should decrease when the anchor graph disagrees with the observed sparse graph. We therefore use the label-free agreement score
\begin{equation}
\label{eq:kappa}
\kappa(A,S_a)=
\frac{\langle A,\PO(S_a)\rangle}
{\|A\|_F\,\|\PO(S_a)\|_F},
\end{equation}
and scale
\begin{equation}
\label{eq:adaptive-mu}
\mu=c\,\kappa(A,S_a)^2
\frac{f_{\rm fit}(U_0,V_0)}{f_{\rm man}(V_0)}.
\end{equation}
The scalar $c$ is selected on a held-out development set and then fixed. This normalization makes $c$ dimensionless and reduces excessive smoothing when local and anchor graphs disagree. We use $c=10$ in all reported experiments.

KL mirror descent is sensitive to coordinates initialized extremely close to zero. Given a common initial hard partition $H$, we use an interior initialization
\begin{equation}
\label{eq:interior-init}
U_0=V_0=(1-\varepsilon)H+\frac{\varepsilon}{r}\ones_n\ones_r^\top,
\end{equation}
with $\varepsilon=0.05$ selected on held-out development data. The manifold-aware factor $V$ is a soft latent representation rather than a coordinate-aligned one-hot membership matrix. We therefore row-normalize $V$ and apply $k$-means for the final discretization. This decoder is fixed before the reported comparisons; a direct row-wise argmax was evaluated separately and was substantially less accurate.

\begin{algorithm}[t]
\caption{Shared-Marginal Mirror Descent (SM-MD)}
\label{alg:smmd-uvg}
\begin{algorithmic}[1]
\Require Sparse affinity $A$, anchor matrix $Z$, rank $r$, manifold weight $\mu$, $\bar\eta>0$, $\rho\in(0,1)$, $\sigma\in(0,1)$.
\State Initialize $U_0,V_0>0$ with $(U_0,V_0)\in\mathcal C$ and set $g_0=U_0^\top\mathbf1_n=V_0^\top\mathbf1_n$.
\For{$k=0,1,\ldots$}
  \State Compute $W_k$ only on $\Omega$ and form the sparse residual $E_k=\mathcal P_\Omega(W_k-A)$.
  \State Compute $G_U^k,G_V^k$ from \eqref{eq:uv-reduced-gradients}, using $L_aV_k=V_k-Z[D_a^{-1}(Z^\top V_k)]$.
  \State Set $\eta\gets\bar\eta$.
  \Repeat
    \State $\bar U\gets U_k\odot\exp(-\eta G_U^k)$, $\bar V\gets V_k\odot\exp(-\eta G_V^k)$.
    \State Solve the $(r-1)$-dimensional strictly convex dual in Theorem~\ref{thm:uv-low-dimensional-kl}.
    \State Obtain the exact feasible projection $(U^+,V^+)$ and set $g^+=\frac12(U^{+\top}\mathbf1_n+V^{+\top}\mathbf1_n)$.
    \If{the Bregman sufficient-decrease test \eqref{eq:uv-bregman-sd} fails}
      \State $\eta\gets\rho\eta$.
    \EndIf
  \Until{\eqref{eq:uv-bregman-sd} holds}
  \State $(U_{k+1},V_{k+1},g_{k+1})\gets(U^+,V^+,g^+)$.
  \If{the stationarity measure \eqref{eq:uv-stationarity-measure} is below tolerance}
    \State \textbf{break}
  \EndIf
\EndFor
\State Return $U,V,g$ and, when a symmetric graph is required, $W_s=(W+W^\top)/2$.
\end{algorithmic}
\end{algorithm}

\subsection{Complexity}
Let $|E|=\operatorname{nnz}(A)$ and let $J$ denote the number of Newton evaluations used by the exact KL projection in one accepted mirror step. The graph is evaluated only on the sparse support $\Omega$. Computing the residual and the products $EV$ and $E^\top U$ costs $O(|E|r)$. The derivative through the shared marginal can be accumulated with $r\times r$ products in $O(nr^2)$ time. With at most $q$ nonzero anchor weights per sample, the manifold action costs $O(nqr)$. Each Newton evaluation of the projection costs $O(nr^2)$, followed by an $O(r^3)$ gauge-fixed solve. Thus one accepted outer iteration has the order
\begin{equation}
O\!\left(|E|r+nqr+J(nr^2+r^3)\right),
\label{eq:complexity-main}
\end{equation}
where the $O(nr^2)$ gradient term is absorbed by the projection term for $J\ge1$. In practice $J$ is small. The working storage, including the sparse input graph and anchor representation, is
\begin{equation}
O\!\left(|E|+n(r+q)\right).
\label{eq:storage-main}
\end{equation}
Neither $W$, $S_a$, nor $L_a$ is materialized as a dense $n\times n$ matrix. When $r$ is large enough that forming the $r\times r$ Hessian is undesirable, \eqref{eq:uv-hvp} supports a matrix-free Newton--CG projection whose curvature products cost $O(nr)$ in the clustering specialization; we retain direct Cholesky in the reported experiments because the clustering ranks are small.

Table~\ref{tab:complexity-compare} places this cost beside representative DS and shared-marginal alternatives. The comparison reports dominant optimization cost after graph construction. For LoRD/B-LoRD, $b$ is the number of inner projection iterations; for low-rank Sinkhorn, $T$ is the number of LR-Dykstra iterations and $d$ is the rank of a factored cost matrix. Input storage is shown when it determines scalability.

\begin{table*}[t]
\centering
\caption{Dominant optimization cost and storage after graph construction. $|E|$ denotes the number of nonzeros in a sparse sample graph. Costs are per outer iteration unless stated otherwise.}
\label{tab:complexity-compare}
\scriptsize
\setlength{\tabcolsep}{4.0pt}
\renewcommand{\arraystretch}{1.15}
\begin{tabular}{p{2.0cm}p{4.2cm}p{3.2cm}p{3.2cm}p{1.8cm}}
\toprule
Method & Dominant time & Feasibility / inner step & Working storage & Dense learned $n\times n$ state? \\
\midrule
SDSGC~\cite{Wang2025SDSGC}
& $O((r+1)n^2)$
& spectral update + ALM graph update
& $O(n^2+nr)$
& Yes \\
DCD~\cite{Yang2016DCD,Lyu2026LoRD}
& $O(|E|r)$ for sparse similarity
& multiplicative MM; no separate matrix projection
& $O(|E|+nr)$
& No \\
LoRD~\cite{Lyu2026LoRD}
& $O(|E|r+nrb_L)$
& Euclidean/Dykstra projection, $b_L$ inner iterations
& $O(|E|+nr)$
& No \\
B-LoRD~\cite{Lyu2026LoRD}
& $O(|E|r+nrb_B)$
& modified projected step, $b_B$ inner iterations
& $O(|E|+nr)$
& No \\
Low-rank Sinkhorn~\cite{Scetbon2021}
& $O(n^2r+Tnr)$ for a general equal-size dense cost; $O(ndr+Tnr)$ for rank-$d$ cost
& LR-Dykstra scaling/projection
& $O(nr)$ factors + cost representation
& No \\
\textbf{SMDSC}
& $O(|E|r+nqr+J(nr^2+r^3))$
& exact $(r-1)$-D Newton projection, $J$ evaluations
& $O(|E|+n(r+q))$
& No \\
\bottomrule
\end{tabular}
\end{table*}

The table makes the scalability claim precise rather than absolute. DCD and LoRD/B-LoRD also avoid a dense learned graph and can be very efficient on sparse inputs; SMDSC does not improve their linear-in-$n$ storage order. Its computational distinction is instead that exact two-factor DS feasibility and the anchor manifold are retained without materializing a dense graph, while the additional feasibility work is confined to a rank-sized Newton system. The direct full-graph SDSGC model pays quadratic memory, whereas low-rank Sinkhorn can become linear in sample size only when the transport cost itself admits a low-rank representation.

\section{Experiments}
\label{sec:experiments}
\subsection{Evaluation protocol}
The model and optimization settings are fixed throughout the reported experiments. Unless otherwise stated, the numerical values in this section come from a single-thread reference implementation that mirrors the stated formulas. The experiments are intended to assess clustering quality, structural feasibility, ablation behavior, and the optimization diagnostics predicted by the theory. Wall-clock values are interpreted only within the common numerical environment used for the corresponding comparison and are not used to claim cross-software speedups.

The main paper uses a \emph{clean matched protocol}. On each trial, methods share the same standardized data, symmetric 10-NN affinity graph, paired random seed, and a common spectral partition mapped into method-compatible initial variables. Method-specific restart advantages are disabled in the main comparison: DCD uses a single matched start, and LoRD/B-LoRD use a single matched feasible start. The matched SDSGC route uses the common sparse support and is therefore labeled \emph{SDSGC-Clean} rather than presented as the original dense solver. A common decoder and matched iteration/time budgets are used whenever the method output permits them. Method-specific native procedures---including DCD prior-assisted starts, the dense/GLAN SDSGC route, and LoRD/B-LoRD multi-start selection---are retained only as supplementary fidelity checks and are excluded from the main ranking and primary scalability conclusions.

The reporting protocol uses paired seeds, the exact shared-marginal Newton projection, and row-normalized $k$-means on $V$ for SMDSC. The main accuracy table uses three paired seeds for the compared methods, while a separate ten-seed analysis evaluates the stability of the fixed SMDSC configuration. The medium benchmark contains Digits1200, Blobs1500, Anisotropic1500, Varied1500, and Imbalanced1800. The large benchmark contains Blobs5000, Varied8000, and Imbalanced10000. We report adjusted Rand index (ARI), normalized mutual information (NMI), clustering accuracy (ACC), feasibility residuals, and anytime ARI. Clean matched and method-specific native procedures are kept separate so that restart policies and preprocessing choices do not enter the main ranking implicitly.

\subsection{Medium-scale comparison}
Table~\ref{tab:medium-dev} reports the matched comparison for the $V$-only shared-marginal model with exact Newton KL projection, adaptive anchor weighting, and row-normalized $k$-means decoding. The proposed method is denoted SMDSC. The comparison uses five datasets and three paired seeds; the separate ten-seed experiment below assesses sensitivity to initialization more directly.

\begin{table*}[t]
\centering
\caption{Medium-scale matched results averaged over three paired seeds.}
\label{tab:medium-dev}
\small
\begin{tabular}{lcccc}
\toprule
Method & ARI & ACC & NMI & RI \\
\midrule
SMDSC & \textbf{0.8308} & \textbf{0.8443} & \textbf{0.9011} & \textbf{0.9493} \\
DS-Sinkhorn & 0.7014 & 0.7251 & 0.8285 & 0.9021 \\
DCD & 0.5840 & 0.6193 & 0.7479 & 0.8512 \\
SC & 0.5658 & 0.6090 & 0.7306 & 0.8438 \\
B-LoRD & 0.5462 & 0.6691 & 0.6845 & 0.8352 \\
LoRD & 0.5346 & 0.6515 & 0.6797 & 0.8400 \\
SDSGC-Clean & 0.2184 & 0.4234 & 0.3445 & 0.6474 \\
\bottomrule
\end{tabular}
\end{table*}

\subsection{Ten-seed stability}
To check whether the fixed configuration depends on a few favorable initializations, we reran SMDSC on the same five medium graphs with ten paired spectral-initialization seeds. Table~\ref{tab:ten-seed} reports the resulting ARI mean and standard deviation. The synthetic Blobs and Anisotropic cases are essentially invariant to initialization, and Varied1500 is also highly stable. Imbalanced1800 shows a wider spread, identifying imbalanced graphs as an important stress case. Across all 50 runs, the maximum row-simplex residual is below $5\times10^{-16}$ and the maximum shared-marginal residual is below $10^{-11}$.

\begin{table}[t]
\centering
\caption{Ten-seed stability of SMDSC.}
\label{tab:ten-seed}
\small
\begin{tabular}{lcc}
\toprule
Dataset & ARI mean & ARI std. \\
\midrule
Digits1200 & 0.6348 & 0.0055 \\
Blobs1500 & 1.0000 & 0.0000 \\
Anisotropic1500 & 0.9984 & 0.0006 \\
Varied1500 & 0.7221 & 0.0011 \\
Imbalanced1800 & 0.7128 & 0.0800 \\
\midrule
Macro mean & 0.8136 & -- \\
\bottomrule
\end{tabular}
\end{table}

\subsection{Large-scale anytime behavior}
Fig.~\ref{fig:anytime} compares ARI under a common time budget in the same numerical environment. SMDSC reaches useful solutions quickly. Its 0--5 s ARI area under the curve (AUC) is 0.7505, compared with 0.7055 for LoRD and 0.5896 for DCD. The practical DS-Sinkhorn variant has the highest AUC (0.7696) and 5 s ARI (0.8186), but it is not the theoretical main algorithm because its post-step scaling can increase the optimized objective. SMDSC instead preserves exact feasibility and the sufficient-decrease theorem at every accepted step. Its ARI peaks around the 2 s checkpoint in this experiment and then decreases while the optimization objective continues to decrease, indicating a remaining objective--clustering alignment issue at later iterations rather than a feasibility or descent failure.

\begin{figure}[t]
\centering
\includegraphics[width=\columnwidth]{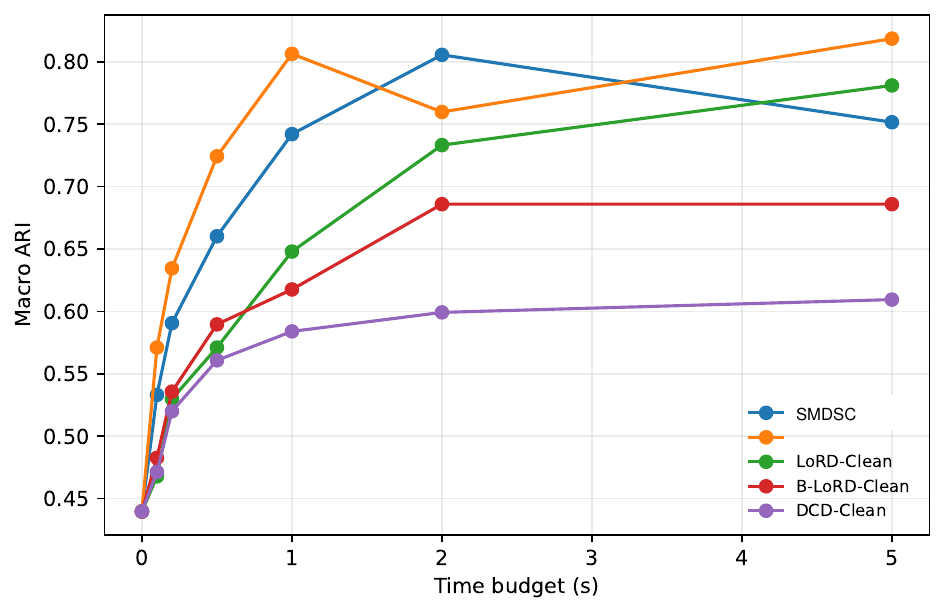}
\caption{Large-scale anytime ARI on three benchmark graphs, averaged over three paired trials in the same numerical environment.}
\label{fig:anytime}
\end{figure}

\begin{table}[t]
\centering
\caption{Large-scale 0--5 s anytime summary.}
\label{tab:large-auc}
\small
\begin{tabular}{lcc}
\toprule
Method & ARI AUC & ARI at 5 s \\
\midrule
DS-Sinkhorn & \textbf{0.7696} & \textbf{0.8186} \\
SMDSC & 0.7505 & 0.7516 \\
LoRD & 0.7055 & 0.7811 \\
B-LoRD & 0.6554 & 0.6859 \\
DCD & 0.5896 & 0.6095 \\
\bottomrule
\end{tabular}
\end{table}

\subsection{Anchor versus Nystr\"om reduced geometry}
The anchor construction is the main method because it satisfies the graph properties in Proposition~\ref{prop:anchor-laplacian} by construction. Nystr\"om can provide a richer low-rank global approximation and achieves higher mean ARI on several benchmark graphs, but its normalized approximation is more sensitive to landmarks and initialization. In a focused Blobs5000 study with ten initializations, the anchor operator has no ARI values below 0.5, whereas the tested Nystr\"om version fails on four of ten trials. Table~\ref{tab:anchor-nys} summarizes the tradeoff.

\begin{table*}[t]
\centering
\caption{Anchor versus Nystr\"om reduced operators under the fixed SMDSC configuration.}
\label{tab:anchor-nys}
\small
\begin{tabular}{llcccccc}
\toprule
Scale & Operator & ARI & NMI & State (MB) & Row resid. & Shared resid. & ARI rank \\
\midrule
Medium & Anchor & 0.7176 & \textbf{0.8457} & \textbf{0.186} & $4.2\!\times\!10^{-16}$ & $9.2\!\times\!10^{-13}$ & 1.80 \\
Medium & Nystr\"om & \textbf{0.7710} & 0.8375 & 0.432 & $4.6\!\times\!10^{-16}$ & $8.9\!\times\!10^{-13}$ & \textbf{1.20} \\
Large & Anchor & 0.7265 & \textbf{0.8723} & \textbf{0.918} & $4.4\!\times\!10^{-16}$ & $1.0\!\times\!10^{-11}$ & 1.67 \\
Large & Nystr\"om & \textbf{0.8408} & 0.8676 & 2.029 & $4.4\!\times\!10^{-16}$ & $1.5\!\times\!10^{-11}$ & \textbf{1.33} \\
\bottomrule
\end{tabular}
\end{table*}

For Blobs5000 over ten initializations, the anchor and Nystr\"om ARI means are 0.7474 and 0.6474, respectively; the corresponding standard deviations are 0.0888 and 0.2039. Nystr\"om has a higher median (0.7526 versus 0.7191) but a lower minimum (0.4259 versus 0.7189), illustrating a higher-ceiling/lower-stability regime. This motivates using anchor geometry in the theoretical main method while retaining Nystr\"om as an extension.

\subsection{Which factor should carry the manifold term?}
For a symmetric affinity and symmetric support, Proposition~\ref{prop:factor-swap} predicts that $V$-only and $U$-only regularization are the same problem after exchanging the factor names. The implementation confirms this exactly when the decoder is exchanged accordingly: their paired ARI values coincide to numerical precision. By contrast, the symmetric $U+V$ manifold model is a genuinely different model. Proposition~\ref{prop:symmetric-one-factor} shows that with $U_0=V_0$ it remains on a one-factor submanifold.

\begin{table}[t]
\centering
\caption{Manifold-side ablation. The symmetric scale $s=0.2$ was selected on independent development data.}
\label{tab:side-ablation}
\small
\begin{tabular}{llcc}
\toprule
Scale & Model & Decoder & ARI \\
\midrule
Medium & $V$-only & $V$ & \textbf{0.8314} \\
Medium & $U$-only & $U$ & \textbf{0.8314} \\
Medium & symmetric, $s=0.2$ & mean & 0.7404 \\
\midrule
Large & $V$-only & $V$ & \textbf{0.7678} \\
Large & $U$-only & $U$ & \textbf{0.7678} \\
Large & symmetric, $s=0.2$ & mean & 0.7151 \\
\bottomrule
\end{tabular}
\end{table}

\subsection{Projection and stationarity diagnostics}
The $(r-1)$-dimensional SPD Newton solver reproduces the previous root-based KL projection to approximately $2.4\times10^{-14}$ in the projected factors, while the shared-marginal residual is around $10^{-12}$ in the validation problem. Fig.~\ref{fig:stationarity} plots the stationarity measure in \eqref{eq:uv-stationarity-measure}. On the Varied1500 diagnostic run, all 60 accepted mirror steps decrease the original objective, and the DS/shared-marginal constraints remain at numerical precision. These diagnostics verify that the implemented line search matches the sufficient-decrease analysis rather than relying on a post-hoc feasibility correction.

\begin{figure}[t]
\centering
\includegraphics[width=\columnwidth]{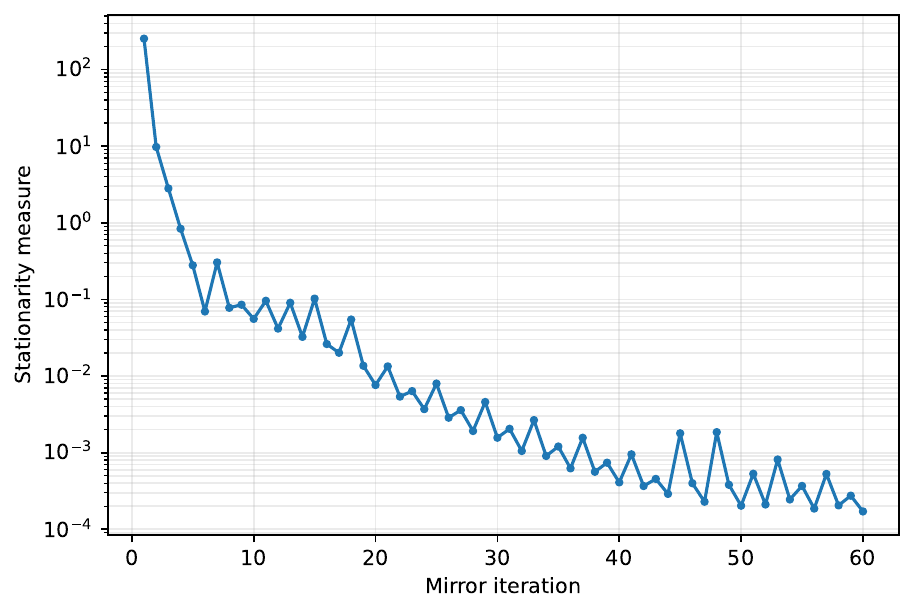}
\caption{Mirror-stationarity diagnostic on Varied1500 for the unified $U,V,g$ implementation.}
\label{fig:stationarity}
\end{figure}

\section{Discussion}
\label{sec:discussion}
The comparison in Tables~\ref{tab:position} and~\ref{tab:complexity-compare} clarifies where SMDSC differs from nearby methods. The advantage is not a uniformly smaller asymptotic order than every low-rank alternative. DCD and LoRD/B-LoRD also avoid a dense learned graph and can be cheaper per iteration. SMDSC instead targets a different combination of properties: two-factor freedom, exact DS feasibility at each accepted iterate, observed-edge sparse fitting, and an explicit reduced manifold term. The additional feasibility cost is confined to a rank-sized Newton system rather than an $n\times n$ constrained update.

This positioning also identifies when the method is most appropriate. If a symmetric one-factor representation is sufficient and manifold regularization is unnecessary, DCD or LoRD/B-LoRD may be simpler and computationally preferable. If the full learned graph itself is the object of interest and $n^2$ storage is acceptable, a direct model such as SDSGC preserves that representation explicitly. SMDSC is aimed at the intermediate regime in which the graph is too large to materialize, the two latent roles should not be tied, and exact stochastic balance is still required during optimization rather than only after it.

The asymmetric factorization deserves one further clarification. The model does not force $W$ to be symmetric, because enforcing $U=V$ would remove the very two-factor flexibility being introduced. For applications requiring a symmetric graph, $W_s=(W+W^\top)/2$ is still nonnegative and doubly stochastic, with rank at most $2r$. The theory and the main clustering decoder operate directly on the factor representation, so symmetrization is a downstream option rather than an optimization constraint.

The current convergence analysis has two explicit boundaries. First, Theorem~\ref{thm:uvg-stationarity} assumes that latent masses do not vanish. This is natural because $D_g^{-1}$ becomes singular as some $g_j\to0$; a vanishing component can instead be removed and interpreted as automatic rank reduction. Second, Corollary~\ref{cor:uvg-kkt} establishes the Euclidean KKT statement only for strictly positive accumulation points. Mirror stationarity remains valid globally over the positive iterates, while a complete boundary KKT result would require additional normal-cone analysis. These conditions should be viewed as the present scope of the theory rather than hidden regularity assumptions.

The anchor construction also represents a deliberate tradeoff. Nystr\"om reduction can attain higher clustering accuracy on favorable datasets, but the row-stochastic anchor operator is nonnegative, exactly stochastic, PSD, and admits a simple spectral bound and matrix-free action. The present method therefore chooses the anchor form for structural control and scalability, not because it dominates every low-rank approximation empirically. A natural extension is to design corrected Nystr\"om operators that preserve these graph properties while retaining their higher empirical ceiling.

Finally, the structural comparison suggests two broader directions. The shared-marginal polytope can support other sparse graph losses beyond squared observed-edge fitting, provided their gradients can be evaluated without forming $W$. Likewise, the low-dimensional KL projection is independent of the particular anchor regularizer. Robust, multi-view, semi-supervised, or adaptive-rank variants can therefore reuse the same feasibility mechanism while changing the data-fit or geometry term.

\section{Conclusion}
\label{sec:conclusion}
We derived an exact rank-space KL projection for shared-marginal low-rank factors with arbitrary positive row marginals of equal total mass. The complete nonlinear projection reduces to a strictly convex $(r-1)$-dimensional dual, with explicit covariance curvature and matrix-free Hessian--vector products. Doubly stochastic clustering is a direct specialization: two nonnegative row-simplex factors share a learned latent mass without being forced to coincide, while sparse observed-edge fitting and a stochastic anchor-reduced Laplacian avoid dense graph construction. Embedded in mirror descent, the projection yields feasible iterates, quantitative Bregman descent, an $O(1/N)$ mirror-stationarity guarantee, and KKT-stationary positive accumulation points under the stated mass condition.

The method does not claim to dominate every low-rank DS solver in raw per-iteration cost. Its contribution is the combination of exact DS structure, two-factor flexibility, sparse graph fitting, and reduced manifold geometry at linear-in-$n$ state for fixed rank, graph degree, and anchor sparsity. The empirical study therefore emphasizes matched structural comparisons and within-environment anytime behavior rather than software-dependent absolute speed claims.

\appendices

\ifdefined\SMDSCFULLPROOFS
\section{Proofs of the auxiliary structural propositions}
\label{app:structural-proofs}

\subsection{Proof of Proposition~\ref{prop:anchor-laplacian}}
\begin{proof}
Symmetry and entrywise nonnegativity follow directly from \eqref{eq:anchor-operator}. Also,
\[
S_a=(ZD_a^{-1/2})(ZD_a^{-1/2})^\top\succeq0.
\]
Since $Z^\top\mathbf1_n=d_a$ and $Z\mathbf1_m=\mathbf1_n$,
\[
S_a\mathbf1_n=ZD_a^{-1}d_a=Z\mathbf1_m=\mathbf1_n.
\]
Thus $S_a$ is symmetric, nonnegative, and stochastic, so its spectral radius equals one. Because $S_a\succeq0$, all eigenvalues lie in $[0,1]$. The claims for $L_a=I_n-S_a$ follow immediately. Finally, $\nabla\phi(V)=\mu L_aV$ and $\|L_a\|_2\le1$, proving the Lipschitz bound.
\end{proof}

\subsection{Proof of Proposition~\ref{prop:factor-swap}}
\begin{proof}
The feasible set is invariant under exchanging $U$ and $V$, and
\[
W(V,U,g)=VD_g^{-1}U^\top=W(U,V,g)^\top.
\]
Symmetry of $A$ and $\Omega$ gives $f_\Omega(W^\top;A)=f_\Omega(W;A)$. The manifold term on $V$ becomes exactly the manifold term on $U$ after the swap. Hence the objectives are identical under factor exchange. The statements about minima and stationary points follow because the swap is a linear bijection of the feasible set.
\end{proof}

\subsection{Proof of Proposition~\ref{prop:symmetric-one-factor}}
\begin{proof}
Suppose $U_k=V_k$. Then $W_k=W_k^\top$, and symmetry of $A$ and $\Omega$ implies $E_k=E_k^\top$. The two fitting gradients are therefore equal. The symmetric manifold regularizer contributes the same gradient to both factors, and the derivative through the shared mass is also identical. Thus $G_U(U_k,V_k)=G_V(U_k,V_k)$. The exponentiated mirror trials are identical. The KL projection problem introduced below is invariant under factor exchange and has a unique primal minimizer; therefore $U_{k+1}=V_{k+1}$. Induction proves the claim.
\end{proof}
\fi

\section{Proof of the exact shared-marginal KL projection theorem}
\label{app:kl-proof}
\begin{proof}[Proof of Theorem~\ref{thm:uv-low-dimensional-kl}]
Write $p_i(z)=u_i(z)/a_i$ and $q_\ell(z)=v_\ell(z)/b_\ell$. These are strictly positive probability vectors. The row-marginal identities in \eqref{eq:uv-softmax-general} are immediate. Differentiating \eqref{eq:uv-dual-potential} gives
\[
\nabla\psi_{a,b}(z)=\sum_i a_i p_i(z)-\sum_\ell b_\ell q_\ell(z)
=c_U(z)-c_V(z).
\]
For a softmax probability vector,
$\nabla p_i=\operatorname{Diag}(p_i)-p_ip_i^\top$, whereas differentiating $q_\ell$ with respect to $z$ gives the negative of the analogous covariance matrix. Hence
\begin{align*}
\nabla^2\psi_{a,b}
&=\sum_i a_i[\operatorname{Diag}(p_i)-p_ip_i^\top]\\
&\quad+\sum_\ell b_\ell[\operatorname{Diag}(q_\ell)-q_\ell q_\ell^\top].
\end{align*}
which is exactly \eqref{eq:uv-dual-hessian}.

For any $x\in\R^r$,
\begin{align*}
&x^\top[\operatorname{Diag}(c_U)-U^\top A_a^{-1}U]x\\
&\quad=\sum_i a_i\left[\sum_j p_{ij}x_j^2-
\left(\sum_j p_{ij}x_j\right)^2\right]\ge0.
\end{align*}
Each summand is a variance under a strictly positive categorical distribution and vanishes iff all coordinates of $x$ are equal. The same holds for the $V$ term, so the Hessian nullspace is exactly $\operatorname{span}\{\ones_r\}$.

Let $M=\ones_n^\top a=\ones_m^\top b$. Under $z\mapsto z+c\ones_r$, the first sum in \eqref{eq:uv-dual-potential} increases by $Mc$ and the second decreases by $Mc$, proving shift invariance. Fixing $z_r=0$ therefore removes the only null direction. For existence, let $z_{\max}=\max_jz_j$ and $z_{\min}=\min_jz_j$. Strict positivity of $\bar U,\bar V$ gives constants $C_U,C_V$ such that
\[
\psi_{a,b}(z)\ge M(z_{\max}-z_{\min})+C_U+C_V.
\]
On the gauge-fixed subspace, $\|z\|\to\infty$ implies $z_{\max}-z_{\min}\to\infty$, so $\psi_{a,b}$ is coercive and has a unique minimizer $z^\star$.

The feasible set of \eqref{eq:uv-kl-projection} contains a strictly positive point, for example
$U_{ij}=a_i/r$ and $V_{\ell j}=b_\ell/r$, whose shared column mass equals $(M/r)\ones_r$. The unique KL minimizer is also strictly positive: if an entry were zero, moving toward this strictly positive feasible point would give one-sided directional derivative $-\infty$ for that scalar KL term while all positive-coordinate derivatives remain finite, contradicting optimality.

Therefore logarithmic KKT equations are valid. With row multipliers $\alpha\in\R^n$, $\beta\in\R^m$, and shared-column multiplier $\lambda\in\R^r$,
\begin{align*}
\log(U/\bar U)+\alpha\ones_r^\top+\ones_n\lambda^\top&=0,\\
\log(V/\bar V)+\beta\ones_r^\top-\ones_m\lambda^\top&=0.
\end{align*}
Setting $z=-\lambda$, the prescribed row constraints recover exactly \eqref{eq:uv-softmax-general}. The remaining KKT condition is $c_U(z)=c_V(z)$, equivalently $\nabla\psi_{a,b}(z)=0$. Strict convexity of the primal KL objective gives uniqueness.
\end{proof}

\ifdefined\SMDSCFULLPROOFS
\section{Proof of the mirror fixed-point lemma}
\label{app:mirror-fixed-proof}
\begin{proof}[Proof of Lemma~\ref{lem:uv-mirror-fixed-point}]
If $\mathcal T_\eta(U,V)=(U,V)$, the first-order condition of \eqref{eq:uv-mirror-map} reduces to the stated normal-cone inclusion. Conversely, if the inclusion holds, then $(U,V)$ satisfies the first-order condition of the mirror subproblem. The KL term is strictly convex in its first argument, so the subproblem has a unique minimizer and the mirror map fixes $(U,V)$. The four KL terms in \eqref{eq:uv-mirror-residual} are nonnegative and vanish simultaneously exactly when $U^+=U$ and $V^+=V$.
\end{proof}

\section{Proofs of the convergence theorem and KKT corollary}
\label{app:convergence-proof}
\begin{proof}[Proof of Theorem~\ref{thm:uvg-stationarity}]
\textbf{Step 1: smoothness and a uniform safe mirror radius.}
For $a>0$, define
\begin{equation}
\mathcal C_a:=\{(U,V)\in\mathcal C:\min_j\bar g_j(U,V)\ge a\}.
\end{equation}
Because every row of $U$ and $V$ lies in the probability simplex, all entries belong to $[0,1]$; the additional lower-mass inequalities are closed and linear in $(U,V)$. Thus $\mathcal C_a$ is compact (and convex). On an open neighborhood of $\mathcal C_{\underline g/2}$, the denominators in $D_{\bar g}^{-1}$ are bounded away from zero. Hence $\Phi$ is twice continuously differentiable there, and compactness gives a finite bound on its Hessian. Therefore there exists a finite Lipschitz constant $L$ for the reduced gradient on this neighborhood, and
\begin{equation}
\label{eq:uv-euclidean-descent}
\begin{aligned}
\Phi(\widetilde U,\widetilde V)\le{}&\Phi(U,V)\\
&+\langle G_U,\widetilde U-U\rangle_F\\
&+\langle G_V,\widetilde V-V\rangle_F\\
&+\frac L2\|\widetilde U-U\|_F^2
+\frac L2\|\widetilde V-V\|_F^2.
\end{aligned}
\end{equation}
whenever the line segment joining the two pairs lies in that neighborhood; in particular this holds for pairs in $\mathcal C_{\underline g/2}$ after choosing the neighborhood around that convex compact set.

For positive simplex rows, every coordinate lies in $(0,1]$. Along the segment between two positive simplex rows, the Hessian of negative entropy is diagonal with entries at least one. Hence negative entropy is one-strongly convex in the Euclidean norm on the positive simplex, and therefore
\begin{equation}
\label{eq:uv-kl-strong-convexity}
\begin{aligned}
&D_{\rm KL}(\widetilde U\|U)+D_{\rm KL}(\widetilde V\|V)\\
&\qquad\ge\frac12\|\widetilde U-U\|_F^2
+\frac12\|\widetilde V-V\|_F^2.
\end{aligned}
\end{equation}
Combining \eqref{eq:uv-euclidean-descent} and \eqref{eq:uv-kl-strong-convexity} shows that the line-search condition is guaranteed whenever the candidate stays in $\mathcal C_{\underline g/2}$ and
\begin{equation}
\eta\le\frac{1-\sigma}{L}.
\end{equation}

It remains to show that sufficiently small mirror steps stay in that region. By compactness,
\[
M:=\max_{(U,V)\in\mathcal C_{\underline g}}
\sqrt{\|G_U(U,V)\|_F^2+\|G_V(U,V)\|_F^2}<\infty.
\]
Let $(U^+,V^+)=\mathcal T_\eta(U,V)$. Comparing the mirror subproblem at $(U^+,V^+)$ with the feasible point $(U,V)$ yields
\begin{align}
&D_{\rm KL}(U^+\|U)+D_{\rm KL}(V^+\|V)\\
&\qquad\le
-\eta\left[\langle G_U,U^+-U\rangle_F+\langle G_V,V^+-V\rangle_F\right]\\
&\qquad\le \eta M\sqrt{\|U^+-U\|_F^2+\|V^+-V\|_F^2}.
\end{align}
Using \eqref{eq:uv-kl-strong-convexity} gives
\begin{equation}
\sqrt{\|U^+-U\|_F^2+\|V^+-V\|_F^2}\le2\eta M.
\end{equation}
For every component $j$,
\begin{align}
|\bar g_j(U^+,V^+)-\bar g_j(U,V)|
&\le\sqrt{\frac n2}\,\Xi,\\
\Xi&:=\sqrt{\|U^+-U\|_F^2+\|V^+-V\|_F^2},\\
|\bar g_j(U^+,V^+)-\bar g_j(U,V)|
&\le \eta M\sqrt{2n}.
\end{align}
Thus the conservative condition
\begin{equation}
\eta\le\eta_{\rm mass}:=\frac{\underline g}{4M\sqrt n}
\end{equation}
(with $\eta_{\rm mass}=+\infty$ when $M=0$) implies
$\eta M\sqrt{2n}\le \underline g/2$ and therefore guarantees $(U^+,V^+)\in\mathcal C_{\underline g/2}$. Therefore every trial with
\begin{equation}
\eta\le\eta_\star:=\min\left\{\eta_{\rm mass},\frac{1-\sigma}{L}\right\}
\end{equation}
is acceptable. Since the line search starts from $\bar\eta$ and reduces by $\rho$, it terminates finitely and every accepted step obeys
\begin{equation}
\eta_k\ge\eta_{\min}:=\rho\min\{\bar\eta,\eta_\star\}>0.
\end{equation}

\textbf{Step 2: positivity and exact feasibility.}
If $(U_k,V_k)>0$, then the exponentiated references in \eqref{eq:uv-exp-trial} are strictly positive. Theorem~\ref{thm:uv-low-dimensional-kl} shows that their exact KL projection is also strictly positive and lies in $\mathcal C$. Positivity and exact feasibility therefore propagate by induction from $(U_0,V_0)$.

\textbf{Step 3: mirror optimality and strong Bregman decrease.}
The first-order variational inequality of the strictly convex mirror subproblem is
\begin{equation}
\label{eq:uv-mirror-vi}
\begin{aligned}
&\left\langle G_U^k+\frac1{\eta_k}(\log U_{k+1}-\log U_k),\widetilde U-U_{k+1}\right\rangle_F\\
&\quad+\left\langle G_V^k+\frac1{\eta_k}(\log V_{k+1}-\log V_k),\widetilde V-V_{k+1}\right\rangle_F\ge0
\end{aligned}
\end{equation}
for every $(\widetilde U,\widetilde V)\in\mathcal C$. Choosing $(\widetilde U,\widetilde V)=(U_k,V_k)$ and using the symmetrized Bregman identity gives
\begin{equation}
\label{eq:uv-sym-kl-optimality}
\begin{aligned}
&\langle G_U^k,U_{k+1}-U_k\rangle_F
+\langle G_V^k,V_{k+1}-V_k\rangle_F\\
&\le-\frac1{\eta_k}D_{\rm KL}(U_{k+1}\|U_k)
-\frac1{\eta_k}D_{\rm KL}(U_k\|U_{k+1})\\
&\quad-\frac1{\eta_k}D_{\rm KL}(V_{k+1}\|V_k)
-\frac1{\eta_k}D_{\rm KL}(V_k\|V_{k+1}).
\end{aligned}
\end{equation}
Substituting \eqref{eq:uv-sym-kl-optimality} into the accepted line-search condition \eqref{eq:uv-bregman-sd} proves \eqref{eq:uv-strong-decrease}.

\textbf{Step 4: summability and asymptotic regularity.}
The sparse squared fitting term and the anchor-Laplacian term are nonnegative, so $\Phi\ge0$. Hence $\{\Phi(U_k,V_k)\}$ is decreasing and convergent. Summing \eqref{eq:uv-strong-decrease} and using $\eta_k\le\bar\eta$ gives
\begin{equation}
\sum_{k=0}^{\infty}
\left[D_{\rm KL}(U_{k+1}\|U_k)+D_{\rm KL}(V_{k+1}\|V_k)\right]<\infty.
\end{equation}
Then \eqref{eq:uv-kl-strong-convexity} yields
\begin{equation}
\sum_{k=0}^{\infty}
\left(\|U_{k+1}-U_k\|_F^2+\|V_{k+1}-V_k\|_F^2\right)<\infty,
\end{equation}
which proves asymptotic regularity.

\textbf{Step 5: non-asymptotic mirror-stationarity rate.}
Since $\sigma<1$, \eqref{eq:uv-strong-decrease} implies
\begin{equation}
\Phi(U_k,V_k)-\Phi(U_{k+1},V_{k+1})
\ge \sigma\eta_k\Delta_k
\ge \sigma\eta_{\min}\Delta_k.
\end{equation}
Summing from $k=0$ to $N-1$ gives
\begin{equation}
\sigma\eta_{\min}\sum_{k=0}^{N-1}\Delta_k
\le\Phi(U_0,V_0)-\Phi(U_N,V_N)
\le\Phi(U_0,V_0)-\Phi_{\inf}.
\end{equation}
The average bounds the minimum, proving \eqref{eq:uv-stationarity-rate}. The same inequality also shows that $\sum_k\Delta_k<\infty$, hence $\Delta_k\to0$.
\end{proof}

\begin{proof}[Proof of Corollary~\ref{cor:uvg-kkt}]
Let $(U_{k_j},V_{k_j})\to(U^\star,V^\star)$ and suppose $U^\star,V^\star>0$. Asymptotic regularity gives $(U_{k_j+1},V_{k_j+1})\to(U^\star,V^\star)$. Because the limit is positive, all entries of these four factor matrices are uniformly bounded away from zero for sufficiently large $j$. The logarithm is therefore locally Lipschitz, and the step-size lower bound implies
\begin{align}
\frac1{\eta_{k_j}}\|\log U_{k_j+1}-\log U_{k_j}\|_F&\to0,\\
\frac1{\eta_{k_j}}\|\log V_{k_j+1}-\log V_{k_j}\|_F&\to0.
\end{align}
The variational inequality \eqref{eq:uv-mirror-vi} is equivalent to
\begin{equation}
\begin{aligned}
-&\Big(G_U^k+\eta_k^{-1}(\log U_{k+1}-\log U_k),\\
&\quad G_V^k+\eta_k^{-1}(\log V_{k+1}-\log V_k)\Big)\\
&\qquad\in N_{\mathcal C}(U_{k+1},V_{k+1}).
\end{aligned}
\end{equation}
The normal-cone mapping of a closed convex set has a closed graph. Passing to the subsequential limit and using continuity of $G_U,G_V$ proves \eqref{eq:uv-normal-stationarity}. Because the accumulation point is positive, the nonnegativity constraints are inactive there. Expanding the normal space generated by the row-sum and shared-column equalities gives \eqref{eq:uv-kkt-positive}.
\end{proof}
\fi

\end{document}